\documentclass[twoside]{article}
\usepackage[accepted]{aistats2023}

\usepackage{aistats2023}

\usepackage{amsmath,amsfonts,bm}

\def\eqref#1{equation~\ref{#1}}

\def\1{\bm{1}}

\def\vmu{{\bm{\mu}}}

\def\vb{{\bm{b}}}

\def\vh{{\bm{h}}}

\def\vw{{\bm{w}}}

\def\vy{{\bm{y}}}

\def\mH{{\bm{H}}}
\def\mI{{\bm{I}}}

\def\mM{{\bm{M}}}

\def\mW{{\bm{W}}}

\def\mY{{\bm{Y}}}

\def\mLambda{{\bm{\Lambda}}}
\def\mSigma{{\bm{\Sigma}}}

\DeclareMathAlphabet{\mathsfit}{\encodingdefault}{\sfdefault}{m}{sl}
\SetMathAlphabet{\mathsfit}{bold}{\encodingdefault}{\sfdefault}{bx}{n}

\newcommand{\R}{\mathbb{R}}

\DeclareMathOperator*{\argmax}{arg\,max}
\DeclareMathOperator*{\argmin}{arg\,min}

\usepackage{microtype}
\usepackage{graphicx}
\usepackage{subfigure}
\usepackage{booktabs}
\usepackage{setspace}
\usepackage{multirow}
\usepackage{color}

\usepackage{amsmath}
\usepackage{amssymb}
\usepackage{mathtools}
\usepackage{amsthm}

\theoremstyle{definition}
\newtheorem{definition}{Definition}
\newtheorem{props}{Proposition}

\usepackage{natbib}
\usepackage[colorlinks,citecolor=blue,urlcolor=blue]{hyperref}
\usepackage{url}
\usepackage{dsfont}
\usepackage{float}
\usepackage{bbding}

\newcommand\blfootnote[1]{%
  \begingroup
  \renewcommand\thefootnote{}\footnote{#1}%
  \addtocounter{footnote}{-1}%
  \endgroup
}

\usepackage[english]{babel}

\begin{document}

%

%

\twocolumn[

\aistatstitle{Inducing Neural Collapse in Deep Long-tailed Learning}

\aistatsauthor{Xuantong Liu$^{1,*}$ \And Jianfeng Zhang$^2$ \And Tianyang Hu$^2$ \And
He Cao$^1$ \And Lujia Pan$^2$ \And Yuan Yao$^{1,}$\Envelope }


\aistatsaddress{$^1$ The Hong Kong University of Science and Technology, $^2$Huawei Noah’s Ark Lab} 


]

\begin{abstract}
Although deep neural networks achieve tremendous success on various classification tasks, the generalization ability drops sheer when training datasets exhibit long-tailed distributions. 
One of the reasons is that the learned representations (i.e. features) from the imbalanced datasets are less effective than those from balanced datasets. 
Specifically, the learned representation under class-balanced distribution will present the \textit{Neural Collapse} ($\mathcal{NC}$) phenomena. $\mathcal{NC}$ indicates the features from the same category are close to each other and from different categories are maximally distant, showing an optimal linear separable state of classification. However, the pattern differs on imbalanced datasets and is partially responsible for the reduced performance of the model.
In this work, we propose two explicit feature regularization terms to learn high-quality representation for class-imbalanced data.
With the proposed regularization, $\mathcal{NC}$ phenomena will appear under the class-imbalanced distribution, and the generalization ability can be significantly improved.
Our method is easily implemented, highly effective, and can be plugged into most existing methods. The extensive experimental results on widely-used benchmarks show the effectiveness of our method. 
\end{abstract}

\section{INTRODUCTION}

\begin{figure}[htbp]
\centering
\subfigure[]{
\begin{minipage}[t]{0.32\linewidth}
\centering
\includegraphics[width=1in]{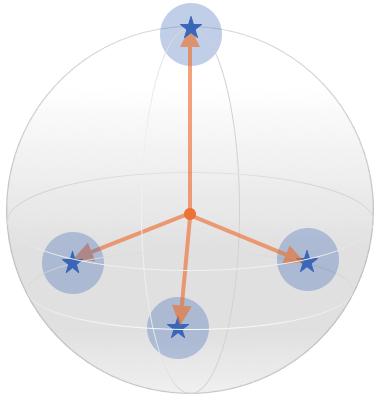}
\end{minipage}%
}%
\subfigure[]{
\begin{minipage}[t]{0.32\linewidth}
\centering
\includegraphics[width=1in]{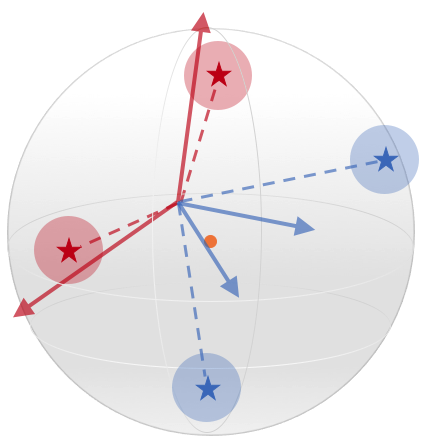}
\end{minipage}%
}%
\subfigure[]{
\begin{minipage}[t]{0.32\linewidth}
\centering
\includegraphics[width=1in]{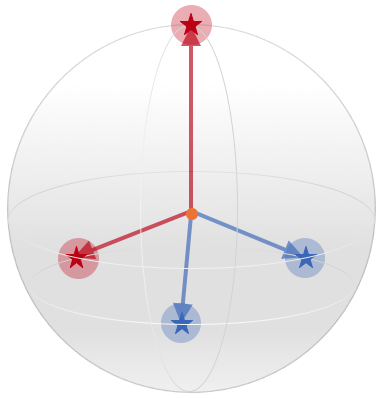}
\end{minipage}
}%
\caption{Illustration of geometry configuration of the zero-centered class means and classifier weights under (a) balanced dataset, (b) imbalanced dataset, and (c) imbalanced dataset with our method.  The arrows represent classifier weights and the stars are the class centers. The size of the circle around the star reflects the variance of the feature from the same class. In (b) and (c), red and blue represent the majority and minority classes, respectively. Note that under imbalanced label distribution, both the centered class means and classifier weights form an asymmetric structure and are no longer parallel.}
 \vspace*{-5mm}
\end{figure}

Modern deep neural networks have shown the ability to outperform humans on many tasks, such as computer vision, natural language processing, playing games, etc., and keep refreshing state-of-the-art performance for complex classification tasks. However, when the training dataset is class-imbalanced, such as a long-tailed distribution, where a few majority classes occupy most of the training samples while a large number of minority classes own very limited samples, the performance of the model drops off a cliff \citep{van2017devil,buda2018systematic}. These imbalanced distributions are ubiquitous in real-world applications, e.g., fault diagnosis, face recognition, autonomous driving, etc. Therefore, how to improve the discriminative ability of the model trained on the imbalanced dataset has always been a topic of considerable concern.
\blfootnote{* This work is done when Xuantong Liu was doing her internship at Huawei Noah’s Ark. \Envelope\ Corresponding author email:  \url{yuany@ust.hk}.
}

Some recent studies focus on learning more effective representation to improve the long-tailed recognition ability. Supervised contrastive loss \citep{khosla2020supervised} is utilized to learn compact within-class and maximally distant between-class representation by introducing uniformly distributed class centers, which leads to improvement in long-tailed performance \citep{li2022targeted, cui2021parametric, zhu2022balancedCL}. 
These characteristics of the feature representation are consistent with those learned from balanced datasets \citep{graf2021dissecting}, where the classification models can spontaneously learn tight and discriminative features. However, contrastive learning is more computationally expensive and requires more iterations to converge than the standard Cross-Entropy (CE) loss.

Meanwhile, the learning behavior of deep classification models in a balanced setting has been investigated both empirically and theoretically \citep{papyan2020prevalence,galanti2021role, han2022neural}. 
The \textit{Neural Collapse} ($\mathcal{NC}$) phenomenon was uncovered by \cite{papyan2020prevalence} when investigating the last-layer embedding, i.e. the feature representation, and the corresponding classifier weights in deep classification models during training.
$\mathcal{NC}$ shows that the learned features (or embedded vectors) of the same class will collapse to their class centers. Meanwhile, these class centers, after globally centered, as well as the classifier weights, will form a simplex equiangular tight frame (ETF) during the terminal phase of training (TPT), i.e. when the model achieves zero training error. The ETF structure maximizes the between-class variability so as the \textit{Fisher discriminant ratio} \citep{fisher1936use}, resulting in an optimal linear separable state for classification. Subsequent studies have found more characteristics of this phenomenon, including the global optimal property \citep{zhu2021geometric} and generalization ability \citep{galanti2021role}. 


However, on imbalanced datasets, the deep neural networks will exhibit different geometric structures, and some $\mathcal{NC}$ phenomena will no longer occur \citep{fang2021exploring,thrampoulidis2022imbalance}. The
last-layer features of the same class still converge to their class means, but the class means, as well as the classifier weights, are not in the form of ETFs any more. 
Specifically, compared to majority classes, the learned features of minority classes will have a larger norm, and correspondingly the norm of classifier weights will be smaller \citep{kang2019decoupling, fang2021exploring}. Furthermore, as the imbalance level increases, the phenomenon of \textit{Minority Collapse} may arise, in which both the learned representations and the classifier weights on minority classes will become indistinguishable \citep{fang2021exploring}. The absence of some $\mathcal{NC}$ property partially explains the performance gap between the balanced and imbalanced datasets. 
 
In this paper, we first elaborate that the appearance of $\mathcal{NC}$ can help to minimize the generalization error in the imbalanced problem. According to this property, we propose two simple yet effective regularization terms to explicitly induce all the $\mathcal{NC}$ phenomena in neural networks trained on imbalanced datasets. The regularization terms can be added to CE loss directly. Compared with supervised contrastive learning, these terms have lower computational cost. Our proposed method not only helps the $\mathcal{NC}$ to occur faster for models trained on the balanced datasets, but also drives the $\mathcal{NC}$ phenomenon to occur on datasets with imbalanced categories. The resulting model can also obtain better generalization ability and robustness without over-training as in \cite{papyan2020prevalence}. Furthermore, our proposed method is orthogonal to most existing methods dealing with long-tailed problems. It thus can be easily plugged into the objective function to obtain further improvements.

In summary, our contributions can be listed below:
\begin{itemize}
    \item
    We observe that when training data is imbalanced, the class centers of minority classes move closer to those of the majority classes, making their instances difficult to distinguish.
    \item We demonstrate that, although some $\mathcal{NC}$ phenomena do not naturally exist in an imbalanced case, we can achieve lower generalization error when all $\mathcal{NC}$ proprieties hold. Thus we propose two simple yet effective regularization terms to manually induce the $\mathcal{NC}$ during imbalanced training.
    \item We experimentally show that our method can significantly improve the performance in various long-tailed tasks and boost most existing methods.
\end{itemize}



\section{PROBLEM SETUP}
\subsection{Preliminaries}
\label{sec:pre}
Let $f_{\phi} \circ g_{\theta}(\cdot)$ denote a neural network classifier, where $g_{\theta}(\cdot)$ is a feature extractor and $f_{\phi}(\cdot)$ is a linear classifier. We define $\mH=[\vh_{1}, \cdots, \vh_{n}]^T\in \R^{n\times P}$ to be the output of $g_{\theta}(\cdot)$. Here $P$ is the dimension of the latent feature, and $n$ is the training sample size. The weights of $f_{\phi}$ are denoted by $\mW=[\vw_1, \cdots,\vw_K] \in \R^{P \times K}$, and the corresponding bias vector is $\vb = [b_1, ..., b_K]$, where $K$ is the number of classes. $\vh_i \in \R^P$ and $y_i \in \{k\}_{k=1}^K$ denote the feature and label of the $i$-th sample.
The label matrix is denoted by $\mY \in \R^{n\times K}$. In the training data, we have $n=\sum_{k=1}^K n_k$, where $n_k$ is the sample size of class $k$. We use $||\cdot||_F$ and $||\cdot||$ to denote the Frobenius norm of a matrix and the $l_2$-norm of a vector.


\begin{definition}[Simplex ETF]
\label{def:ETF}
A simplex ETF is a collection of equal-length and maximally-equiangular vectors. We call a $P\times K$ matrix $\mM$ an ETF if it satisfies
\begin{equation}
    \mM^{\mathrm{T}} \mM=\alpha\left(\frac{K}{K-1} \mI - \frac{1}{K-1}\mathds{1}_K\mathds{1}_K^\mathrm{T}\right)
\end{equation}
for some non-zero scalar $\alpha$. Where $\mI$ is the  identity matrix, and $\mathds{1}_K$ is an all-ones vector.
\end{definition}

Let $\vmu_k=\frac{1}{n_k}\sum_{y_i=k} \vh_{i}$ be the center of class $k$ and $\vmu_C=\frac{1}{K}\sum_{k=1}^{K} \vmu_k$ be the arithmetic mean of the class centers. In the balanced case, where we have $n_k=\frac{n}{K}$ for each class, $\mathcal{NC}$ will appear during TPT. The phenomena can be formally described by four properties: 
\begin{itemize}
    \item $(\mathcal{NC}_1)$ \textbf{Variability collapse.} Intra-class variances collapse to zero during the terminal phase of training, i.e., for any sample $i$ from class $k$, we have
    \begin{equation}
    ||\vh_{i}-\vmu_k||=0
    \end{equation}
\end{itemize}
\begin{itemize}
    \item $(\mathcal{NC}_2)$ \textbf{Convergence to simplex ETF.} The class centers (after zero-center normalization) converge to the vertices of an ETF, i.e.
\begin{gather}
    \cos (\vmu_k-\vmu_C, \vmu_{k'}-\vmu_C) = -\frac{1}{K-1},\\
    ||\vmu_k-\vmu_C||=||\vmu_{k'}-\vmu_C||.
\end{gather}
\end{itemize}
\begin{itemize}
    \item $(\mathcal{NC}_3)$ \textbf{Convergence to self-duality.} The weights of linear classifiers are parallel to the corresponded zero-centered class centers, i.e. 
\begin{equation}
    \vw_k= \alpha (\vmu_k-\vmu_C).
\end{equation}
\end{itemize}
\begin{itemize}
    \item $(\mathcal{NC}_4)$ \textbf{Simple decision rule.} Given a feature, the last-layer classifier's behavior is equivalent to the nearest class center (NCC) decision rule, i.e.
\end{itemize}
\begin{equation}
    \argmax\limits_{k}\langle \vw_k, \vh \rangle = \argmin \limits_{k} ||\vh-\vmu_k||
\end{equation}

\subsection{Neural Collapse and Imbalanced Data}
In this section, we first illustrate why $\mathcal{NC}$ disappears on imbalanced datasets using mean squared error (MSE) loss. 
Then we demonstrate that the $\mathcal{NC}$ properties will lead to a lower generalization error bound thus we can benefit from it under imbalanced distribution. 

\subsubsection{The Optimal Classifier Under Imbalanced Distribution}
Some recent studies have shown that the test performance of neural networks trained with MSE loss is comparable to those trained with CE loss in classification tasks \citep{demirkaya2020exploring, fang2021exploring, hu2021understanding, han2022neural}. Thanks to its tractability, we can use MSE loss to illustrate the absence of the $\mathcal{NC}$ phenomenon on imbalanced datasets. For linear classifiers, the MSE loss is
\begin{equation}
    \mathcal{L} (\mH, \mW)=\frac{1}{2n}||\mY - (\mH \mW + \mathds{1}_n\vb^\mathrm{T})||_F^2.
    \label{LS1}
\end{equation}
Let $\bar{\vh}=\frac{1}{n}\sum_{i=1}^n \vh_i$ be the global feature mean,
$\mSigma_T = (\mH-\mathds{1}_n\bar{\vh})^\mathrm{T}(\mH-\mathds{1}_n\bar{\vh})$ be the total covariance matrix of $\mH$ and $\dot{\mM}=[\vmu_1-\bar{\vh},...,\vmu_K-\bar{\vh}] \in \R^{P\times K}$. We can have the closed form of the optimal $\mW$ and $\vb$ under the MSE loss as follows.
\begin{props}
\label{prop:W_LS}
\citep{webb1990optimised}. In general, for fixed features $\mH$, the optimal weight matrix and the bias vector that minimize $\mathcal{L}$($\mH,\mW$) are
\begin{gather}
\mW_{LS} =\mSigma_T^\dagger \dot{\mM}\mLambda, \\
\vb_{LS} = \frac{1}{n}\mathds{1}_n^\mathrm{T}\mY - \vmu_G\mW_{LS},
\label{W_LS}
\end{gather}
\textit{where $^\dagger$ denotes the Moore-Penrose pseudoinverse,  and $\mLambda=\text{diag}(n_1, \cdots, n_K)$ is a diagonal matrix.}
\end{props}

From Eq.(\ref{W_LS}), we can observe that the optimal weight matrix depends on the features and is strongly affected by $\mLambda$, i.e. the proportions of classes. Specifically, the classifier weights of the majority classes will have larger norms. The $\mathcal{NC}$ phenomena reflect the intimate connection between the last layer features and the classifier weights. Thus, skewed classifiers imply that the features are also biased, and many studies have empirically investigated that the uneven label distribution can lead to an imbalanced feature space \citep{kang2020exploring,fang2021exploring,li2022targeted}.

Particularly, \cite{fang2021exploring} show the \textit{Minority Collapse} phenomenon that reveals the skewed classifier weights encountering an imbalanced label distribution where majority classes own much more samples than the minority ones. In addition, they theoretically prove that unbiased classifiers can be obtained through over-sampling. However, empirical results show a limited performance improvement or even decline due to the over-fitting of the minority classes \citep{drummond2003c4,weiss2007cost}. 
On the other hand, the classifiers are always better tuned than the learned features \citep{thrampoulidis2022imbalance}.
Therefore, in this work, we mainly focus on regularizing the embeddings during training to get non-skewed and representative features. Then we tune a balanced classifier based on our well-learned features.



\subsubsection{Importance of $\mathcal{NC}$ on Imbalanced Datasets}
As we already know that when the training set is class-imbalanced, the geometric structure of the classifiers and centered class means are not symmetric, which may introduce some bias in the model and affect the performance of the test set \citep{kang2019decoupling, kang2020exploring,fang2021exploring}. Recent work indicates that compact within-class representations along with evenly distributed class centers can help learn high-quality representations, and substantial practice confirms this \citep{li2022targeted,zhu2022balancedCL,cui2022generalized}. These intuitions lead to similar situations with $\mathcal{NC}$. In this section, we explain why the $\mathcal{NC}$ can be considered favorable representations and can provide reduced generalization errors under long-tailed distributions from the perspective of \textit{domain adaptation}.

As a standard evaluation approach in long-tailed learning, models are usually tested on balanced datasets. Since the training set is imbalanced, we can regard this scenario as a label shift domain adaptation problem, where the source domain is imbalanced, and the target domain is balanced. 

First, the following proposition shows that properties of $\mathcal{NC}_1$ and $\mathcal{NC}_2$ can be approximately preserved in the target domain.

\begin{props}
\label{prop:NCgeneralize}
\citep{galanti2021role}
Let $\mu^S_k$ (resp. $\mu^T_k$) and $\sigma^{S}_{k}$ (resp. $\sigma^{T}_{k}$) be the mean and variance of the representations of class $k$ on the source domain (resp. target domain). For any two different classes, $k$ and $k'$, with probability at least $1-\delta$ over $\mathcal{D}_S$, we have
\begin{equation}
\label{(eq:NCgener}
\frac{\sigma^{T}_{k} + \sigma^{T}_{k'}}{2 \| \mu^T_k - \mu^T_{k'} \|^2} \le (1+A^2) \bigg( \frac{\sigma^{S}_{k} + \sigma^{S}_{k'}}{2 \| \mu^S_k - \mu^S_{k'} \|^2} + B \bigg),
\end{equation}
\textit{where $A=\frac{\mathcal{O}(\sqrt{\log(1/\delta)/n_k})}{||\mu^T_k-\mu^T_{k'}||},
B=\frac{\mathcal{O}(\sqrt{\log(1/\delta)/n_k})}{||\mu^S_k-\mu^S_{k'}||^2}$.}
\end{props}
The ETF geometry of $\{\mu_k\}_{k=1}^K$ indicates that the distance $\| \mu^S_k - \mu^S_{k'} \|$ achieve maximum value for all $k\neq k'$. On the other hand, $\| \mu^T_k - \mu^T_{k'} \|$ is also lower bounded by $\| \mu^S_k - \mu^S_{k'} \|$ \citep{galanti2021role}. Hence, $A$ and $B$ are upper bounded and diminish to zero as $n_k$ gets larger. Therefore, we can roughly speak $\mathcal{NC}_1$ and $\mathcal{NC}_2$ can generalize to the target domain.

We then illustrate how the existence of $\mathcal{NC}_1$ and $\mathcal{NC}_2$ help to reduce the generalization error. According to \cite{ben2006analysisDA}, for any classifier $h$, the error on target domain $\epsilon_T(h)$ will be bound by the empirical error on the source domain and the divergence between source and target feature domains plus a constant:
\begin{equation}
\label{eq: da_bound}
    \epsilon_T(h) \leq \hat{\epsilon}_S (h) + d_\mathcal{H} (\mathcal{D}_S^Z, \mathcal{D}_T^Z) + \textrm{const},
\end{equation}
where $d_\mathcal{H} (\mathcal{D}_S^Z, \mathcal{D}_T^Z)$\footnote{$d_\mathcal{H} (\mathcal{D}_S^Z, \mathcal{D}_T^Z)$ denotes the $\mathcal{H}$-divergence between $\mathcal{D}_S^Z$ and $\mathcal{D}_T^Z$, a precise definition is provided in \cite{ben2010theory}.} measures some `distance' between source and target domains over the feature space $\mathcal{Z}$. 
Although not exactly the same, substituting $d_\mathcal{H}$ with the Jensen–Shannon distance $d_{JS}$ \citep{endres2003new} will not significantly change the result. Theoretically, minimizing $d_{JS}$ between source and target distributions will reduce the right-hand side of Eq.(\ref{eq: da_bound}) as well. Let $\mathcal{D}^Z$ and $\mathcal{D}^Y$ be the distributions defined over the latent feature space and label space, respectively. As $\mathcal{D}^Y$ can be induced from $\mathcal{D}^Z$ from a generative perspective, according to \cite{zhao2019invariant_DA}, we have 
\begin{equation}
\label{feature_dis}
    d_{JS}(\mathcal{D}_S^Z, \mathcal{D}_T^Z) \geq d_{JS}(\mathcal{D}_S^Y, \mathcal{D}_T^Y),
\end{equation}
i.e., $d_{JS}(\mathcal{D}_S^Z, \mathcal{D}_T^Z)$ is the lower bounded by $d_{JS}(\mathcal{D}_S^Y, \mathcal{D}_T^Y)$, which is a constant determined by source and target label distributions. 

With $\mathcal{NC}_1$ and $\mathcal{NC}_2$, the distribution over $\mathcal{Z}$ collapses to a $K$-component mixture Dirac distribution. More precisely, we have $\Pr(Z = \vh)=\Pr(Y = \vy)$. In this case, $d_{JS}(\mathcal{D}_S^Z, \mathcal{D}_T^Z)$ attains its lower bound $d_{JS}(\mathcal{D}_S^Y, \mathcal{D}_T^Y)$, which is the objective of some classical domain adaptation algorithms \citep{pmlr-v37-long15, DBLP:series/acvpr/GaninUAGLLML17},

\section{LEARNING REPRESENTATION VIA INDUCING NEURAL COLLAPSE}

The previous analysis inspires us to induce $\mathcal{NC}$ phenomena to imbalanced training. We mainly focus on the core properties, $\mathcal{NC}_1$ and $\mathcal{NC}_2$, and come up with two corresponding regularization terms.

\subsection{Feature Regularization}
\paragraph{Compact within-class features.}
$\mathcal{NC}_1$ underlines that the model is seeking to learn compact within-class features by pushing the last-layer embedding to be close to their class centers, which seems natural but actually hard to achieve in practice. \cite{han2022neural} decomposed the MSE loss and discovered that the loss in the late training stages is dominated by the $\ell_2$-distance between the feature and the corresponding class center. This indicates that although $\mathcal{NC}_1$ is the inevitable trend, it is quite difficult to realize.
Therefore, we add explicit regularization to make $\mathcal{NC}_1$ more inclined to appear.
Especially, for the class-imbalanced dataset, we consider the inverse ratio of class sizes as weights to avoid excessive force on the majority classes. This indicates the difference between our $\mathcal{NC}_1$ regularization and the \textit{center loss} \citep{wen2016discriminative} that pushes all features equally to their class center. Formally, we define the $\mathcal{NC}_1$ regularization as the within-class feature distance, $\mathcal{L}_{W}$, with the formula of
\begin{equation}
    \mathcal{L}_{W} = \sum_{k=1}^K\sum_{y_i=k} \frac{1}{n_k}||\vh_{i}-\vmu_k||_2^2.
\end{equation}
\paragraph{Distinct between-class features.}
$\mathcal{NC}_2$ shows that with balanced class distribution, all pairs of centered class means tend to form equal-sized angles, implying the maximally separated between-class features. However, under the imbalanced distribution, the class centers of the minority classes are close to the majority ones, leading to indistinguishable features. Therefore, we propose $\mathcal{NC}_2$ regularization to minimize the maximal pairwise cosine similarity between all the centered class means, equivalent to maximizing the minimal pairwise angle. Consider the angular version, the objective of $\mathcal{NC}_2$ regularization is:
\begin{equation}
    \max \min \limits_{k \neq k'} \arccos \frac{\langle\dot{\vmu}_k, \dot{\vmu}_{k'}\rangle}{||\dot{\vmu}_k|| \cdot ||\dot{\vmu}_{k'}||},
\end{equation}
where $\dot{\vmu}_k=\vmu_k-\vmu_C$.
As noted in \cite{wang2020mma}, updating the average of each vector's maximum cosine is more efficient than just optimizing the global maximum cosine. Therefore, we define the formula for the $\mathcal{NC}_2$ regularization as
\begin{equation}
    \mathcal{L}_{B} = -\frac{1}{K}\sum_{k=1}^K \min \limits_{k',k' \neq k}\arccos \frac{\langle\dot{\vmu}_k, \dot{\vmu}_{k'}\rangle}{||\dot{\vmu}_k|| \cdot ||\dot{\vmu}_{k'}||}.
\end{equation}
In summary, our proposed feature regularization includes two terms, $\mathcal{L}_{W}$ and $\mathcal{L}_{B}$, corresponding to minimize the within-class distance and maximize the between-class discrepancy, respectively. They can be easily coupled with supervised losses with a linear classifier to regularize the penultimate layer embedding. Finally, we have the following loss for training:
\begin{equation}
    \mathcal{L} = \mathcal{L}_{sup} + \lambda_1 \mathcal{L}_{W} + \lambda_2 \mathcal{L}_{B}.
\end{equation}
where $\mathcal{L}_{sup}$ denotes the supervised loss, \textit{e.g.} CE loss and MSE loss. $\lambda_1$ and $\lambda_2$ are hyperparameters that control the impact of $\mathcal{L}_{W}$ and $\mathcal{L}_{B}$. 
\subsection{Occurrence of Neural Collapse}
First, we illustrate that $\mathcal{L}_{B}$ will lead all pairs of the $K$ class means to have the same cosine equals to $-\frac{1}{K-1}$, with the following proposition.
\begin{props}
\label{props:max_angle}
\citep{wang2020mma} The minimum of the maximal pair-wise cosine similarity between $n$ vectors is $\frac{-1}{n-1}$, which can be reached when the vectors have an equal-sized pair-wise angle and zero mean.
\end{props}
Therefore, denote $\hat{\mM}=[\frac{\dot{\vmu}_1}{||\dot{\vmu}_1||},...,\frac{\dot{\vmu}_K}{||\dot{\vmu}_K||}]$. Recall that the objective of $\mathcal{L}_B$ is to minimize the maximal pair-wise cosine similarity of the centered class means, thus with $\mathcal{L}_B$, we have 
\begin{equation}
    \hat{\mM}^\mathrm{T}\hat{\mM} = \frac{K}{K-1}\mI-\frac{1}{K-1}\1_K \1_K^\mathrm{T}.
\end{equation}

According to \textbf{Definition} \ref{def:ETF}, $\hat{\mM}$ form a simplex ETF. Furthermore, although $\mathcal{L}_{W}$ and $\mathcal{L}_{B}$ do not explicitly enforce the centered class means to have an equal norm, we empirically observe this desired result (see the experimental result in \textbf{Section} \ref{sec:ft_analyze}). 
Let $||\dot{\vmu}_1||= \cdots = ||\dot{\vmu}_K||=\alpha$ and $\bar{\mM}=[\dot{\vmu}_1, \cdots ,\dot{\vmu}_K]$, then we have
\begin{equation}
    \bar{\mM}^\mathrm{T}\bar{\mM} = \alpha  \left(\frac{K}{K-1}\mI-\frac{1}{K-1}\1_K \1_K^\mathrm{T}\right),
\end{equation}
indicating the centered class means indeed from an ETF.
Therefore, with the proposed feature regularization terms $\mathcal{L}_{W}$ and $\mathcal{L}_{B}$, $\mathcal{NC}_1$ and $\mathcal{NC}_2$ can happen even when the training set is imbalanced. 

In addition, we can prove that with the existence of $\mathcal{NC}_1$ and $\mathcal{NC}_2$, retrain the classifier with class-balance sampling, the classifier can become parallel with the centered feature mean, indicating the self-duality ($\mathcal{NC}_3$). Ultimately, the symmetric structure of the regularized class means brings about an unbiased linear classifier.

\begin{props}
\label{prop:NC12-3}
\textbf{Proposition} \ref{prop:W_LS}+$\mathcal{NC}_{1}$+$\mathcal{NC}_{2}$+class-balanced sampling can lead to $\mathcal{NC}_{3}$.
\end{props}
\begin{proof}
With class-balanced sampling, the training label distribution can be regarded as balanced, and $\dot{\mM}=\bar{\mM}$. Then the optimal re-trained classifier $\mW_{r}$ is
\begin{equation}
    \mW_{r} =\frac{n}{K} \mSigma_T^\dagger \dot{\mM}, 
\end{equation}
with the existence of $\mathcal{NC}_1$, we have  $\mSigma_T=\dot{\mM}\dot{\mM}^\mathrm{T}$. Thus,
\begin{equation}
\begin{aligned}
    \mW_{r} & =\frac{n}{K}(\dot{\mM}\dot{\mM}^\mathrm{T})^\dagger \dot{\mM} \\
    & = \frac{n}{K}(\dot{\mM}\dot{\mM}^\mathrm{T})^\dagger \dot{\mM}\dot{\mM}^\mathrm{T} (\dot{\mM}^\mathrm{T} )^\dagger \\
    & = \frac{n}{K}(\dot{\mM}^\mathrm{T} )^\dagger, \nonumber
\end{aligned}
\end{equation}
with $\mathcal{NC}_2$ which implies that $\dot{\mM}$ form a simplex-ETF, thus, $(\dot{\mM}^\mathrm{T} )^\dagger=c \dot{\mM}$ for some constant $c$ \citep{papyan2020prevalence}, then we can obtain $\mW_{r} = \alpha \dot{\mM} \nonumber$, demonstrating the asserted self-duality ($\mathcal{NC}_3$).
\end{proof}
In conclusion, with the proposed $\mathcal{L}_W$ and $\mathcal{L}_B$, we can obtain \textit{compact within-class} and \textit{distinct between-class} representations under imbalanced-class distribution. In line with linear discriminant analysis (LDA) \citep{fisher1936use}, this provides an optimal solution for the linear classifier.

\section{EXPERIMENTS}
\subsection{Classification and long-tailed recognition}
In this section, we conduct various experiments on image classification tasks on both balanced and long-tailed datasets to validate the effectiveness of our method. We denote our approach as NC, indicating the occurrence of the $\mathcal{NC}$ phenomena. By default, CE is adopted as $\mathcal{L}_{sup}$.

\subsubsection{Experiment Setup}
\paragraph{Datasets.}
Two balanced datasets (CIFAR10 and CIFAR100) and three long-tailed datasets (CIFAR10-LT, CIFAR100-LT, and ImageNet-LT) are used in our experiments. Following \cite{cao2019learning},  CIFAR10/100-LT are created by downsampling each class's samples to obey an exponential decay with an imbalance ratio $r=100$ and $10$. Here $r=\max\{n_k\} / \min \{n_k\}$. ImageNet-LT \citep{liu2019large}, including 115,846 samples and 1,000 categories with size ranging from 5 to 1,280, is generated from the ImageNet-2012 \citep{deng2009imagenet} dataset using a Pareto distribution with the power value $\alpha = 6$. 

\paragraph{Baselines.} 
In addition to the typical approaches for addressing imbalanced data, such as re-sampling (RS) and re-weighting (RW) in inverse proportion to the class size, the investigation of more conducive methods that decouple representation learning and classifier training, as well as relevant methods inspired by $\mathcal{NC}$, are also carried out.
To be specific, we compare traditional supervised learning methods with DRW \citep{cao2019learning}, LWS \citep{kang2019decoupling}, and cRT \citep{kang2019decoupling}, and two recent works, namely BBN \citep{zhou2020bbn} and MiSLAS \citep{zhong2021improving}. Our comparison also includes supervised contrastive learning approaches, namely FCL \citep{kang2020exploring}, KCL \citep{kang2020exploring}, and TSC \citep{li2022targeted}.  In addition, the comparison involves $\mathcal{NC}$-inspired methods such as ETF classifier+DR \citep{yang2022we} and ARB-Loss \citep{xie2023neural}.

\paragraph{Implementation details.}
We mainly follow the common training protocol. In all experiments, we adopt SGD optimizer with the momentum of 0.9, weight decay of 0.005, and train the model for 200 epochs following \cite{alshammari2022long}. 
We utilize mix-up \citep{zhang2018mixup} during the representation learning stage for all datasets.
For CIFAR10/100(-LT), we use ResNet-32 \citep{he2016deep} as the backbone and a multi-step schedule that decays the learning rate as its 0.1 at the 160-th and 180-th epochs with initialization of 0.1. 
We use 4 GeForce GTX 2080Ti GPUs with a batch size of 128.
For ImageNet-LT, we use ResNeXt-50 \citep{xie2017aggregated} as the backbone and cosine schedule that gradually decays the learning rate from 0.05 to 0. We use 4 Tesla V100 GPUs to train the models with a batch size of 256. 
We also adopt Randaugment \citep{cubuk2020randaugment} for ImageNet-LT. 
We report the average results of three independent trials with different random seeds. Our code is available at \href{https://github.com/Pepper-lll/NCfeature}{https://github.com/Pepper-lll/NCfeature}.

The hyperparameters $\lambda_1$ and $\lambda_2$ need to be adjusted according to the complexity of the datasets. In general, simple datasets with few categories require a small magnitude of feature regularization, while for complex datasets with plenty of categories, we need larger $\lambda_1$ and $\lambda_2$. Besides, similar to \cite{li2022targeted}, we also find that it is better to regularize the feature learning from half of the training process for large-scale datasets, i.e., CIFAR100 and ImageNet-LT. Our hyperparameter settings and the epoch number to start feature regularization are summarized in Table \ref{Tab:hyper_setting}. 

The class centers $\{\vmu_k\}_{k=1}^K$ are updated in each mini-batch, instead of in the entire training set, which has been proved not efficient in large-scale datasets \citep{wen2016discriminative}. Besides, our regularization terms are better to combine with re-balancing strategies to ensure the matching between classifier weights and class centers. The combination can lead to a remarkable improvement. In our experiments, we choose DRW and cRT as the re-balancing strategies.

\begin{table}[ht]
\small
\centering
\caption{Hyperparameter setting.}
\label{Tab:hyper_setting}
\vspace{1mm}
\begin{tabular}{c|ccc}
\toprule
Dateset & $\lambda_1$ & $\lambda_2$ & start epoch\\
\midrule
CIFAR10(-LT)  & 0.01 & 0.1 & 0 \\
CIFAR100(-LT) & 0.01 & 0.5 & 100 \\
ImageNet-LT   & 0.05 & 1.0 & 100  \\
\bottomrule
\end{tabular}
\vspace{-3mm}
\end{table}


\subsubsection{Results}
\paragraph{Balanced data.}
As we mentioned before, our method is applicable to both balanced and imbalanced datasets. First, we conduct experiments to validate our model on balanced CIFAR10 and CIFAR100 datasets. Table \ref{balance_res} shows that our method can reduce the generalization error with both CE and MSE loss.

\paragraph{Imbalanced data.}
Table \ref{cifar-table} and \ref{imgnet-lt} present our results on CIFAR10-LT, CIFAR100-LT, and ImageNet-LT. We can find that our method surpasses existing methods on all three datasets. For ImageNet-LT, we further test the accuracy on three groups of classes according to the sample size, including Many-shot ($>$100 samples), Medium-shot (20$\sim$100 samples), and Few-shot ($<$20 samples). The results show that our method can substantially improve the accuracy of the Medium- and Few-shot categories with almost no impact on the accuracy of the Many-shot categories compared to the plain training with CE.

\begin{table}[ht]
\small
\vspace{-2mm}
\centering
\caption{Top-1 test accuracy (\%) on the balanced datasets.}
\label{balance_res}
\vspace{2mm}
\begin{tabular}{cccc}
\toprule
Method & CIFAR10 & CIFAR100  \\
\midrule
CE & \textbf{93.4} & 71.8   \\
+NC & 93.3 &\textbf{72.1}   \\ \midrule
MSE & 91.1 & 70.7  \\
+NC & \textbf{91.7} & \textbf{71.9}  \\
\bottomrule
\end{tabular}
\vspace{-3mm}
\end{table}

\begin{table}[ht]
\small
\vspace*{-3mm}
\centering
\caption{Top-1 test accuracy (\%) on CIFAR10-LT and CIFAR100-LT. The results of the compared methods are obtained
from their respective original papers. The best and second-best results are marked in bold and underlined.}
\label{cifar-table}
\vspace{2mm}
\begin{tabular}{ccccc}
\toprule
Method & \multicolumn{2}{c}{CIFAR10-LT} & \multicolumn{2}{c}{CIFAR100-LT} \\ \midrule
imbalance ratio & 100 & 10 & 100 & 10 \\
\midrule
CE & 70.4 & 86.4 & 38.4 & 55.7 \\
CE-RS & 72.8 & 87.8 & 36.7 & 57.7 \\
CE-RW & 74.4 & 87.9 & 32.5 & 58.2 \\
CE-DRW & 75.1 & 86.4 & 42.5 & 56.2 \\
LDAM-DRW & 77.0 & 88.2 & 43.5 & 58.7 \\ 
BBNm& 79.9 & 88.4 & 42.6 & 59.2 \\
MiSLAS & 82.1 & 90.0 & 47.0 & \underline{63.2} \\ \midrule
KCL & 77.6 & 88.0 & 42.8 & 57.6 \\
TSC & 79.7 & 88.7 & 43.8 & 59.0 \\ \midrule
ETF classifier+ DR & 76.5 & 87.7 & 45.3 & - \\
ARB-Loss & \textbf{83.3} & \textbf{90.2} & 47.2 & 62.1 \\ \midrule
NC-DRW & 81.9 & \underline{89.8} & \underline{48.6} & 63.1 \\
NC-DRW-cRT & \underline{82.6} & \textbf{90.2} & \textbf{48.7} & \textbf{63.6}\\
\bottomrule
\end{tabular}
\vspace{-3mm}
\end{table}

\begin{table}[ht]
\small
\centering
\caption{Top-1 test accuracy (\%) on ImageNet-LT.}
\label{imgnet-lt}
\vspace{2mm}
\begin{tabular}{ccccc}
\toprule
Methods & Many & Medium & Few & All \\ \midrule
CE & \textbf{68.2} & 38.1 & 5.82 & 45.3 \\
CE-RS & 64.6 & 42.6 & 17.8 & 47.8 \\
CE-RW & 52.0 & 41.4 & 19.8 & 42.5 \\
CE-DRW & 52.6 & 45.7 & 31.5 & 46.4 \\
CE-cRT & 58.8 & 33.0 & 26.1 & 47.3 \\
CE-LWS & 57.1 & 45.2 & 29.3 & 47.7 \\
MiSLAS & 61.7 & \textbf{51.3} & \textbf{35.8} & 52.7 \\ \midrule
FCL & 61.4 & 47.0 & 28.2 & 49.8 \\
KCL & 62.4 & 49.0 & 29.5 & 51.5 \\
TSC & 63.5 & 49.7 & 30.4 & 52.4 \\  \midrule
ETF classifier+ DR & - & - & - & 44.7 \\
ARB-Loss & 60.2 & 51.8 & 38.3 & 52.8 \\
\midrule
NC-DRW & \underline{67.1} & 49.7 & 29.0 & \underline{53.6} \\
NC-DRW-cRT & 65.6 & \underline{51.2} & \underline{35.4} & \textbf{54.2}\\ 
\bottomrule
\end{tabular}
\vspace{-3mm}
\end{table}

\begin{table}[ht]
\small
\centering
\caption{Top-1 test accuracy (\%) on real-world long-tail datasets of our methods combined with others. Note that we replicated experiments of RIDE with data distributed parallel training and got results with slight differences from \cite{wang2020long}. c10, c100 and iNet are short for CIFAR10, CIFAR100 and ImageNet respectively.}
\label{tb:combine}
\vspace{2mm}
\begin{tabular}{cccc}
\toprule
Method & c10-LT & c100-LT & iNet-LT \\ \midrule
LDAM-DRW & 77.0 & 42.0 & 48.8 \\
+NC & 77.1(0.1$^\uparrow$) & 43.2(1.2$^\uparrow$) & 49.5(0.7$^\uparrow$) \\
Logit Adjust & 77.4 & 43.9 & 51.1 \\
+NC & 78.8(1.4$^\uparrow$) & 44.6(0.7$^\uparrow$) & 53.2(2.1$^\uparrow$) \\
RIDE (2 experts) & - & 46.5 & 51.9 \\
RIDE (3 experts) & - & 47.5 & 54.2 \\
RIDE (4 experts) & - & 48.8 & 55.2 \\
+NC (2 experts) & - & 46.8(0.3$^\uparrow$) & 52.2(0.3$^\uparrow$) \\
+NC (3 experts) & - & 48.1(0.6$^\uparrow$) & 54.8(0.6$^\uparrow$) \\
+NC (4 experts) & - & 49.1(0.3$^\uparrow$) & 56.0(0.8$^\uparrow$)\\
\bottomrule
\end{tabular}
\vspace{-3mm}
\end{table}

\begin{table*}[ht]
\small
\centering
\caption{Random Noise Robustness Results. CE$^\ddag$  denotes Cross-Entropy loss with the feature regularization $\mathcal{L}_{W}+\mathcal{L}_{B}$.}
\vspace{2mm}
\label{Gnoise}
\begin{tabular}{cc|ccccc|ccccc}
\toprule
\multicolumn{2}{c|}{Gaussian noise std} & 0.00 & 0.10 & 0.20 & 0.30 & 0.40 & 0.00 & 0.10 & 0.20 & 0.30 & 0.40 \\ \midrule
\multicolumn{1}{c|}{Dataset} & Loss & \multicolumn{5}{c|}{ResNet-32} & \multicolumn{5}{c}{ResNet-18} \\ \hline
\multicolumn{1}{c|}{\multirow{2}{*}{CIFAR10}} & CE & \textbf{93.3} & 89.4 & 76.6 & 54.7 & 35.7 & 94.9 & 91.5 & \textbf{79.2} & 56.9 & 35.9 \\
\multicolumn{1}{c|}{} & CE$^\ddag$ & 93.1 & \textbf{89.6} & \textbf{76.6} & \textbf{57.5} & \textbf{39.0} & \textbf{95.1} & \textbf{91.8} & 78.9 & \textbf{57.1} & \textbf{37.4} \\ \hline
\multicolumn{1}{c|}{\multirow{2}{*}{CIFAR10-LT}} & CE & 77.0 & 75.0 & 62.3 & 45.3 & 32.5 & 79.2 & 75.7 & 63.5 & 47.2 & 33.8 \\
\multicolumn{1}{c|}{} & CE$^\ddag$ & \textbf{79.2} & \textbf{75.5} & \textbf{63.2} & \textbf{47.8} & \textbf{35.5} & \textbf{81.0} & \textbf{77.5} & \textbf{66.5} & \textbf{51.4} & \textbf{37.9} \\ \hline
\multicolumn{1}{c|}{\multirow{2}{*}{CIFAR100}} & CE & 71.8 & \textbf{60.2} & \textbf{40.0} & 23.9 & 14.1 & 78.2 & 65.9 & 44.2 & 25.0 & 14.3 \\
\multicolumn{1}{c|}{} & CE$^\ddag$ & \textbf{72.3} & 59.0 & 39.7 & 23.9 & \textbf{15.1} & \textbf{78.6} & \textbf{67.8} & \textbf{46.9} & \textbf{28.2} & \textbf{16.6} \\ \hline
\multicolumn{1}{c|}{\multirow{2}{*}{CIFAR100-LT}} & CE & 42.5 & 37.2 & 25.4 & 16.3 & 10.2 & 46.8 & 41.1 & 31.6 &\textbf{23.6}  & \textbf{17.6} \\
\multicolumn{1}{c|}{} & CE$^\ddag$ & \textbf{45.7} & \textbf{39.0} & \textbf{27.2} & \textbf{16.8} & \textbf{11.1} & \textbf{47.2} & \textbf{41.6} & 31.6 & 23.3 & 16.3 \\ \bottomrule
\end{tabular}
\vspace{-2mm}
\end{table*}

\paragraph{Combine with existing approaches.}
Our regularization terms can be easily plugged into most of the existing algorithms. To validate the effectiveness, in Table \ref{tb:combine}, we add the proposed regularization terms to three different types of algorithms. We follow their original experiment settings to compare the performance differences before and after adding regularization terms. The results show that our regularization terms can increase the accuracy in all three algorithms.


\vspace{-2mm}
\begin{table}[ht]
\small
\begin{center}
\caption{\textbf{Ablation studies} on the effectiveness of each regularization term on CIFAR10/100-LT. Note that we apply DRW for all experiments here.}
\label{Tab:ablation}
\vspace{2mm}
\begin{tabular}{ccccc}
\toprule
Method & \multicolumn{2}{c}{CIFAR10-LT} & \multicolumn{2}{c}{CIFAR100-LT} \\ \midrule
imbalance ratio & 100 & 10 & 100 & 10 \\
\midrule
CE      & 75.1 & 86.4 & 42.4 & 56.2 \\
+Centor Loss  & 78.7 & 89.1 & 46.3 & 61.2 \\
+$\mathcal{L}_W$    & 79.1 & 88.1 & 46.9 & 61.3 \\
+$\mathcal{L}_B$    & 80.1 & 88.6 & 47.6 & 61.7 \\ 
+Centor Loss \&$\mathcal{L}_B$  & 77.5 & 89.2 & 46.5 & 61.4 \\
+$\mathcal{L}_W$\&$\mathcal{L}_B$ & \textbf{81.9} & \textbf{89.8} & \textbf{48.6} & \textbf{63.1} \\
\bottomrule
\end{tabular}
\end{center}
\end{table}

\subsection{Discussions}
In this section, to verify the correctness and further explore the properties of our method, we show the learned representations, performance robustness, and ablation study on various combinations of loss and regularizations.

\subsubsection{Representation Analysis}
\label{sec:ft_analyze}
We extensively analyze the representations learned with our method to explain the advantages relative to the baseline. As for the corresponding analysis of classifiers, we obtained consistent findings with previous studies \citep{kang2019decoupling} and therefore do not repeat them here.

\paragraph{Maximally separated class centers.}
We compare the pair-wise angles of the centered class means learned on CIFAR10-LT with vanilla training, re-sampling (RS), re-weighting (RW), and the proposed regularization terms in Figure \ref{fig:pair-wise-angle}. We arrange the class indexes in descending order based on their sizes. 
Under a long-tailed distribution, the minority class centers move closer to the majority with plain model. In Figure \ref{fig:angle-LT}, the angles between class 8 and 0, class 9 and 1, and class 5 and 3 are around $50^\circ$ which is far lower than the optimal angle of $96^\circ$. 
RS and RW can assist in the acquisition of more distinguishable features, as demonstrated Figure \ref{fig:angle-LT-RS} and \ref{fig:angle-LT-RW}). However, with our regularization terms (Figure \ref{fig:angle-LT-nc}), we can observe that the pair-wise angles between all the class centers remain consistently close to the optimum value. In addition, the significant improvement on the experimental results indicates that the features learned by our method are more generalizable.

\begin{figure*}[ht]
\centering
\subfigure[vanilla]{
\label{fig:angle-LT}
\begin{minipage}{3.8cm}
\centering
\includegraphics[width=3.6cm]{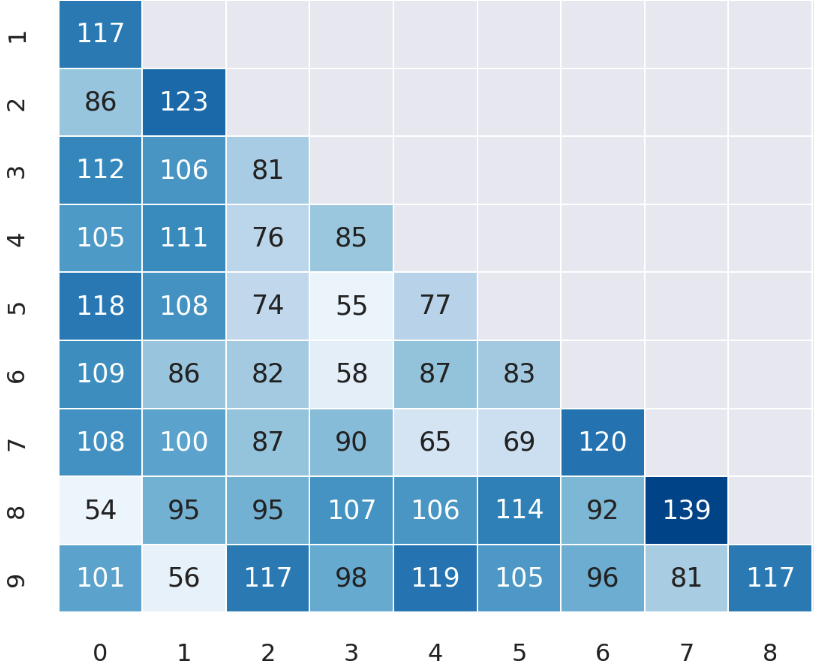}
\end{minipage}%
}%
\subfigure[w/ RS]{
\label{fig:angle-LT-RS}
\begin{minipage}{3.8cm}
\centering
\includegraphics[width=3.6cm]{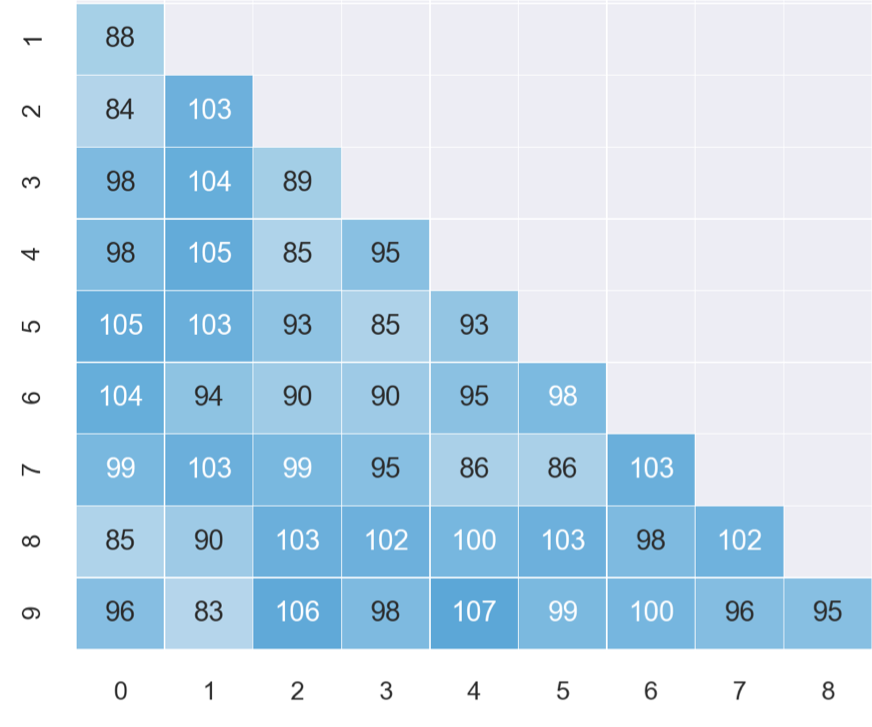}
\end{minipage}%
}%
\subfigure[w/ RW]{
\label{fig:angle-LT-RW}
\begin{minipage}{3.8cm}
\centering
\includegraphics[width=3.6cm]{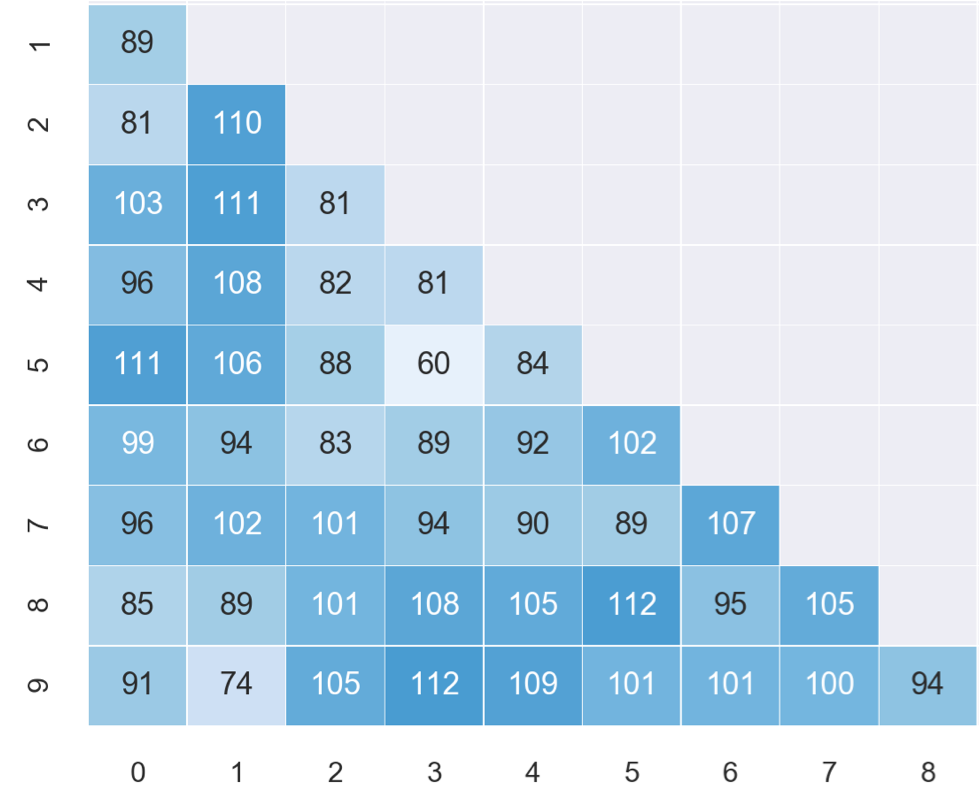}
\end{minipage}%
}%
\subfigure[w/ $\mathcal{L}_W$ and $\mathcal{L}_B$]{
\begin{minipage}{4.4cm}
\label{fig:angle-LT-nc}
\centering
\includegraphics[width=4.1cm]{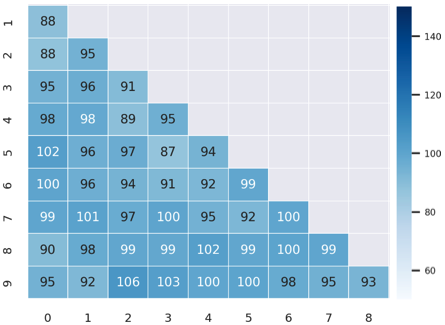}
\end{minipage}
}
\vspace{-2mm}
\caption{Pair-wise angle degree between centered class means trained on CIFAR10-LT. Note that the optimal pair-wise angle for 10 classes is $\arccos \frac{-1}{10-1} \approx 96.4^\circ$.}
\label{fig:pair-wise-angle}
\vspace{-2mm}
\end{figure*}

\begin{figure}[ht]
\centering
\subfigure[CIFAR10(-LT)]{
\begin{minipage}[t]{0.33\linewidth}
\centering
\includegraphics[width=1.21in]{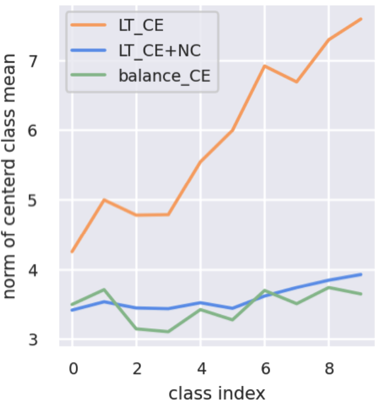}
\end{minipage}%
}%
\subfigure[CIFAR100(-LT)]{
\begin{minipage}[t]{0.66\linewidth}
\centering
\includegraphics[width=1.9in]{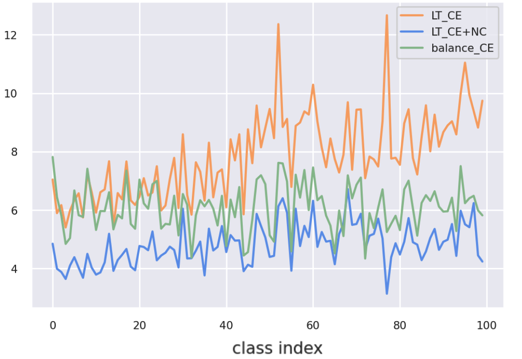}
\end{minipage}%
}%
\centering
\caption{The norm of centered class means on a balanced dataset, long-tailed dataset w/ and w/o inducing NC is represented in different colors. Note that the class index is inversely sorted by the sample size.}
\label{fig:center-norm}
\vspace{-2mm}
\end{figure}

\paragraph{Zero-centered class means with the equal norm.}
Although neither $\mathcal{L_W}$ nor $\mathcal{L}_B$ forces the class center to be of equal norm, we can observe it in our experiments, as shown in Figure \ref{fig:center-norm}. This result strongly indicates that we can successfully induce $\mathcal{NC}$ in imbalanced data. 

\subsubsection{Robustness}
We test the robustness of our method against random noise with different neural networks on CIFAR10/100 and their long-tailed version where the imbalance ratio $r=100$. Here Resnet-32 and ResNet-18 are employed. ResNet-18 is a wider network with the last-layer feature dimension of 512, while ResNet-32 is 64. The models are all trained with DRW. The results are reported in Table \ref{Gnoise}. We can observe that our regularization terms can improve the robustness for different model capacities.

\subsubsection{Ablation Study}
We conduct experiments to examine the effectiveness of two regularization terms separately over CE loss and the comparison with \textit{center loss}. 
The results, presented in Table~\ref{Tab:ablation}, demonstrate that each term can significantly improve accuracy individually, and that their combination produces the best results.
Meanwhile, $\mathcal{L}_W$ consistently produces better results than \textit{center loss}, suggesting that modifying the coefficient is crucial.
We can also find that $\mathcal{NC}_2$ property is more useful, implying the importance of sufficiently distant class centers for the long-tail recognition task.

\section{RELATED WORK}
\subsection{Long-tailed recognition} 
Long-tailed distribution is ubiquitous in the real world, which brings big challenges for most deep learning models. Classical methods dealing with this problem include data re-sampling and loss re-weighting. The former refers to re-sampling the instances to achieve relatively balanced training data, basically including over-sampling \citep{ando2017deepoversample,shelke2017reviewresample}, under-sampling \citep{shelke2017reviewresample}, and class-balanced sampling \citep{cui2019effectivenum}. 
Instead of changing the original data distribution, loss re-weighting uses cost-sensitive re-weighting strategies and assigns different weights to instances from different classes according to the sample sizes \citep{lin2017focal,cui2019effectivenum}.
However, although the re-sampling and re-weighting approaches can improve the performance of minority classes, they may lead to overfitting \citep{li2022targeted} and hurt the representation learning \citep{kang2019decoupling}.

Recent works also focus on representation learning under long-tailed data distribution. This stream of study mainly follows a two-stage training scheme that decouples the representation and classifier learning \citep{kang2019decoupling, zhong2021improving, li2022targeted, kang2020exploring, zhu2022balancedCL}. \cite{kang2019decoupling} observed that a high-quality representation requires fully utilizing the training instances equally, while a re-balancing technique is crucial for an unbiased classifier. 
On the other hand, some works take advantage of the superior representation learning ability of contrastive loss to extract the feature for deep long-tailed learning; then train a classifier upon the feature extractor with cost-sensitive loss or class-balanced sampling \citep{kang2020exploring,li2022targeted,zhu2022balancedCL}. Supervised contrastive learning shows 
superiority in representation learning under imbalanced distribution and achieves SOTA for long-tailed recognition tasks \citep{li2022targeted, zhu2022balancedCL,cui2022generalized}. However, these methods usually converge slowly and require complex network structures compared to traditional supervised learning.

Researchers also explored methods based on the ensemble. They usually utilize multiple models over different data distributions \citep{wang2020long} or perform representation learning and classifier training with separate branches \citep{zhou2020bbn,zhu2022balancedCL}.
This kind of approach is generally considered to be orthogonal to the single-model approach described above.

\subsection{Neural collapse}

A recent study \citep{papyan2020prevalence} discovered the phenomenon named \textit{Neural Collapse} ($\mathcal{NC}$), stating that the last-layer embedding and classifiers will converge to a symmetric geometry named simplex \textit{Equiangular Tight Frame} (ETF) for deep classifiers trained on balanced data. A more precise description of the $\mathcal{NC}$ phenomena is delivered in Section \ref{sec:pre}. Subsequent studies indicate that $\mathcal{NC}$ will eventually occur, independent of the loss function, the optimizer, batch-normalization, and regularization, as long as the training data exhibits a balanced distribution \citep{zhu2021geometric,han2022neural,kothapalli2022neural}. 
Meanwhile, the intrinsic merit of $\mathcal{NC}$ has also been revealed, including ensuring global optimality, stronger generalization and robustness, and transferability \citep{papyan2020prevalence, zhu2021geometric, galanti2021role}. 

The investigation of $\mathcal{NC}$ has also been carried over to the imbalanced data case, where different phenomena are uncovered.
\cite{fang2021exploring} demonstrated that the minority classifiers have smaller pair-wise angles than the majority ones and will even merge together as the imbalance level increases, named \textit{Minority Collapse}. This phenomenon provides some reason of the performance drop. \cite{thrampoulidis2022imbalance} provides a general frame that is equivalent to ETF for balanced data, and reveals an asymmetric geometry of the last-layer feature and classifiers for imbalanced distribution. Furthermore, the perfect alignment between the class feature means and classifiers vanished under the imbalanced distribution. However, \cite{thrampoulidis2022imbalance} illustrates the general geometry with a special encoding framework and does not discuss whether this geometry with an imbalanced dataset has merit or defect.

Inspired by the $\mathcal{NC}$ phenomenon, some researchers have attempted to improve the model's classification ability encountering imbalanced distribution by eliminating \textit{Minority Collapse}, including fixing the classifier as an ETF \citep{yang2022we} and adjusting the CE loss \citep{xie2023neural}. 
Distinct from these works, our work analyzes that obtaining high-quality features is the key to the improvement and thus proposes regularization to guide learning representations.

\section{CONCLUSIONS}
In this paper, we argue that the existence of $\mathcal{NC}$ is crucial for long-tailed recognition and propose two simple but effective regularization terms to induce the appearance of $\mathcal{NC}$. We empirically show that under the imbalanced data distribution, the class centers of minority classes are close to the majority ones, leading to the overlap among different classes over the feature space and confusion of the classifier. With our method, the deep classification models are able to learn \textit{compact within-class} and maximally \textit{distinct between-class} features. Extensive experiments confirm that our method can enhance the generalization power of the deep classification model, especially when the training set is imbalanced. Our method is more efficient than contrastive loss based methods, and we set new state-of-the-art performance for single model based methods on widely used benchmarks. 

Our proposed regularization guides the representation learning to be of `optimal' geometry for classification, which is particularly beneficial for training sets with imbalanced labels. However, the learned geometry is validated empirically and lacks complete theoretical guarantees, leading to manually tuning the related hyperparameters. 
In the future, we plan to formally analyze the geometry obtained with our regularization and provide some theoretical justification for the choice of hyperparameters.

\section*{Acknowledgements}
This work was supported in part by National Natural Science Foundation of China / Research Grants Council (NSFC/RGC) Joint Research Scheme Grant N\_HKUST635/20, Hong Kong Research Grant Council (HKRGC) Grant 16308321, ITF UIM/390, as well as awards from Smale Institute of Mathematics of Computation. This research
made use of the computing resources of the X-GPU cluster supported by the HKRGC Collaborative Research Fund C6021-19EF.


\bibliography{ref}
\bibliographystyle{ref}


\end{document}